\documentclass[letterpaper, 10 pt, conference]{ieeeconf}  %

\IEEEoverridecommandlockouts   %
\usepackage{amsmath,amssymb,amsfonts}
\usepackage{mathtools}
\usepackage{bm}
\usepackage{graphicx}
\usepackage{booktabs}
\let\labelindent\relax
\usepackage{enumitem}
\usepackage{xcolor}
\usepackage[font=small,labelfont=bf,belowskip=0pt]{caption}
\usepackage{url}
\usepackage[hidelinks,bookmarks=false]{hyperref}

\usepackage[
  backend=biber,
  style=numeric-comp,   %
  sortcites=true,
  sorting=none,         %
  natbib=true,          %
  giveninits=true,
  maxcitenames=4,       %
  minbibnames=3,        %
  maxbibnames=4,        %
  doi=false,            %
  url=false,
  hyperref=true
]{biblatex}
\DeclareSourcemap{
  \maps[datatype=bibtex]{
    \map{
      \step[fieldset=file,      null]
      \step[fieldset=abstract,  null]
      \step[fieldset=keywords,  null]
      \step[fieldset=urldate,   null]
      \step[fieldset=month,     null]
      \step[fieldset=day,       null]
      \step[fieldset=language,  null]
      \step[fieldset=isbn,      null]
      \step[fieldset=issn,      null]
      \step[fieldsource=doi,
            match=\regexp{^\s*https?://(dx\.)?doi\.org/},
            replace={}]
      \step[fieldsource=booktitle,
            match=\regexp{^\s*Proceedings\s+of\s+(the\s+)?},
            replace={}]
      \step[fieldsource=eventtitle,
            match=\regexp{^\s*Proceedings\s+of\s+(the\s+)?},
            replace={}]
      \step[fieldsource=booktitle,
            match=\regexp{^\s*(19|20)\d{2}\s+},
            replace={}]
      \step[fieldsource=eventtitle,
            match=\regexp{^\s*(19|20)\d{2}\s+},
            replace={}]
      \step[fieldsource=booktitle,
            match=\regexp{^\s*\d+(st|nd|rd|th)\s+},
            replace={}]
      \step[fieldsource=eventtitle,
            match=\regexp{^\s*\d+(st|nd|rd|th)\s+},
            replace={}]
    }
  }
}
\renewbibmacro{in:}{}

\DeclareFieldFormat{title}{\mkbibemph{#1}}
\DeclareFieldFormat[article,incollection,report,techreport,inproceedings,book,misc]{title}{\mkbibquote{#1}}

\DeclareFieldFormat{eprint:arxiv}{arXiv\addcolon\space#1}
\usepackage{balance}
\usepackage{fontawesome5}  %

\usepackage{algorithm}
\usepackage[noend]{algpseudocode}
\algrenewcommand\algorithmicrequire{\textbf{Input:}}
\algrenewcommand\algorithmicensure{\textbf{Output:}}
\makeatletter
\algrenewcommand\ALG@beginalgorithmic{\small}
\algnewcommand\parState[1]{\State\parbox[t]{\dimexpr\linewidth-\ALG@thistlm\relax}{%
  \rightskip=0pt plus 2em\hyphenpenalty=10000\strut#1\strut}}
\makeatother
\algnewcommand\Stage[1]{\Statex\rule{0pt}{\dimexpr\ht\strutbox+3pt\relax}\textit{#1}}

\usepackage[acronym,nonumberlist,nogroupskip,nomain]{glossaries}
\glsdisablehyper

\graphicspath{{figures/}}

\usepackage{amsthm}

\theoremstyle{plain}
\newtheorem{theorem}{Theorem}
\newtheorem{proposition}{Proposition}

\theoremstyle{definition}
\newtheorem{definition}{Definition}
\newtheorem{assumption}{Assumption}

\theoremstyle{remark}

\newif\ifdraft
\draftfalse

\definecolor{draftblue}{HTML}{1F4E79}
\definecolor{drafttodo}{HTML}{C0392B}

\usepackage{comment}
\newenvironment{ptsinner}
  {\color{draftblue}\begin{itemize}[leftmargin=1.2em,itemsep=1pt,topsep=2pt,parsep=0pt]}
  {\end{itemize}}
\ifdraft
  \newenvironment{pts}{\begin{ptsinner}}{\end{ptsinner}}
\else
  \excludecomment{pts}
\fi

\newcommand{\todo}[1]{\ifdraft\textcolor{drafttodo}{\textbf{[TODO: #1]}}\fi}

\newcommand{\para}[1]{\smallskip\noindent\textbf{#1.}}

\newcommand{\subpara}[1]{\smallskip\noindent\emph{#1.}}

\newsavebox{\tabbox}

\usepackage{dblfloatfix}

\newcommand{\attrib}{\alpha}

\definecolor{darkgreen}{HTML}{326C89}
\definecolor{orange}{HTML}{D16D3B}
\definecolor{grey}{HTML}{6D7784}
\definecolor{purple}{HTML}{754FB8}
\definecolor{green}{HTML}{4A895C}
\definecolor{jhublue}{HTML}{002D72}

\newcommand{\StoC}{\textbf{\textcolor{darkgreen}{S2C}}}
\newcommand{\ET}{\textbf{\textcolor{orange}{ET}}}
\newcommand{\nom}{\textbf{\textcolor{grey}{Nom}}}
\newcommand{\CPO}{\textbf{\textcolor{purple}{CPO}}}
\newcommand{\Lag}{\textbf{\textcolor{green}{Lag}}}
\newcommand{\ETsf}{{\ET}\dag}
\newcommand{\nomsf}{{\nom}\dag}
\newcommand{\CPOsf}{{\CPO}\dag}
\newcommand{\Lagsf}{{\Lag}\dag}

\newcommand{\wh}{\widehat}

\DeclareMathOperator*{\expect}{\mathbb{E}}

\DeclareMathOperator{\expl}{expl}          %
\newcommand{\shield}{\text{\tiny{\faShield*}}}

\newcommand{\reals}{\mathbb{R}}

\newcommand{\state}{s}                      %
\newcommand{\statespace}{\mathcal{S}}
\newcommand{\failset}{\mathcal{F}}

\newcommand{\action}{\ctrl}                     %
\newcommand{\actionSet}{\mathcal{A}}

\newcommand{\transkernel}{P}                %
\newcommand{\discount}{\gamma}
\newcommand{\tdisc}{t}

\newcommand{\iagent}{{i}}
\newcommand{\ego}{e}

\newcommand{\turn}{\sigma}                  %

\newcommand{\outcome}{J}                    %
\newcommand{\reward}{r}                     %
\newcommand{\consFunc}{g}                   %
\newcommand{\consTot}{g_{\mathrm{tot}}}     %
\newcommand{\safeSet}{\Omega}               %
\newcommand{\safeSetMax}{{\safeSet^{*}}}

\newcommand{\valFunc}{V}
\newcommand{\qFunc}{Q}

\newcommand{\policy}{\pi}
\newcommand{\policySet}{\Pi}

\newcommand{\fallback}{\policy^{\shield}}
\newcommand{\actionSafe}[1]{\actionSet^{#1}_{\mathrm{safe}}}
\newcommand{\policySafe}[1]{\policySet^{#1}_{\mathrm{safe}}}
\newcommand{\policyAll}[1]{\policySet^{#1}}
\newcommand{\policyExecd}[1]{\policy^{#1}_{\mathrm{exec}}}

\newcommand{\safetyFilter}{\phi}

\newcommand{\game}{\mathcal{G}}
\newcommand{\gameSC}{\game_{\mathrm{SC}}}   %
\newcommand{\gameFilt}{\game_{\safetyFilter}}
\newcommand{\ctrl}{a}

\usepackage{ifthen}
\newboolean{include-notes}
\newboolean{include-new}
\newboolean{include-remove}
\setboolean{include-notes}{true}
\setboolean{include-new}{false}
\setboolean{include-remove}{false}

\usepackage{xcolor}
\usepackage[normalem]{ulem}
\newcommand{\haimin}[1]{\ifthenelse{\boolean{include-notes}}{\textcolor{magenta}{\textbf{Haimin:} #1}}{}}
\renewcommand{\todo}[1]{\ifthenelse{\boolean{include-notes}}{\textcolor{teal}{\textbf{TODO:} #1}}{}}
\newcommand{\remove}[1]{\ifthenelse{\boolean{include-remove}}{\textcolor{red}{\sout{#1}}}{}}
\newcommand{\new}[1]{\ifthenelse{\boolean{include-new}}{\textcolor{purple}{#1}}{#1}}
\newcommand{\ruihan}[1]{\ifthenelse{\boolean{include-notes}}{\textcolor{blue}{\textbf{Ruihan:} #1}}{}}

\usepackage{lipsum}

\DeclareDocumentEnvironment{example}{}{\noindent\textbf{Running Example:}\itshape}{}

\newcommand{\iter}{k}                %
\newcommand{\costBudget}{b}          %
\newcommand{\evalPool}{\mathcal{B}}  %
\newcommand{\numEp}{n}               %
\newcommand{\poolProb}{p}            %
\newcommand{\oppPool}{\mathcal{P}}   %

\newcommand{\nashTol}{\epsilon}

\newacronym{mdp}{MDP}{Markov decision process}
\newacronym{mg}{MG}{Markov game}
\newacronym{scmg}{SC-MG}{Safety-Critical Markov Game}
\newacronym{cmdp}{CMDP}{constrained Markov decision process}
\newacronym{ne}{NE}{Nash equilibrium}
\newacronym{br}{BR}{best response}
\newacronym{NE}{NE}{Nash equilibrium}

\newacronym{hj}{HJ}{Hamilton--Jacobi}
\newacronym{hji}{HJI}{Hamilton--Jacobi--Isaacs}
\newacronym{cbf}{CBF}{control barrier function}
\newacronym{qcbf}{Q-CBF}{Q-control barrier function}
\newacronym{lrsf}{LR-SF}{least-restrictive safety filter}
\newacronym{rss}{RSS}{responsibility-sensitive safety}
\newacronym{ef}{EF}{Exclusive Fault}
\newacronym{ua}{UA}{Unilateral Avoidance}
\newacronym{mv}{MV}{Mutual Viability}

\newacronym{rl}{RL}{reinforcement learning}
\newacronym{marl}{MARL}{multi-agent reinforcement learning}
\newacronym{dnn}{DNN}{deep neural network}
\newacronym{ppo}{PPO}{proximal policy optimization}
\newacronym{ippo}{IPPO}{independent proximal policy optimization}
\newacronym{cpo}{CPO}{Constrained Policy Optimization}
\newacronym{sac}{SAC}{soft actor--critic}
\newacronym{isaacs}{ISAACS}{iterative soft adversarial actor--critic for safety}

\newacronym{mpc}{MPC}{model predictive control}
\newacronym{sf}{SF}{safety filter}

\usepackage{eso-pic}

\newcommand{\david}[1]{\ifthenelse{\boolean{include-notes}}{\textcolor{teal}{\textbf{David:} #1}}{}}

\title{\LARGE \bf Turning Safety into Competence: Minimally Exploitable Robot Policies via Safety-Filtered Reinforcement Learning}

\author{Ruihan Wu$^{1,*}$, Rui Yang$^{1,*}$,
Donggeon David Oh$^{2}$, Duy P.\ Nguyen$^{2,\dagger}$,
and Haimin Hu$^{1,\dagger}$}

\begin{document}

\twocolumn[{%
\renewcommand\twocolumn[1][]{#1}%
\maketitle
\AddToShipoutPictureFG*{%
  \AtPageLowerLeft{%
    \put(\LenToUnit{0.5\paperwidth},\LenToUnit{0.35in}){%
      \makebox(0,0)[b]{%
        \parbox{\textwidth}{%
          \centering
          \normalfont\scriptsize
          This work has been submitted to the IEEE for possible
          publication. Copyright may be transferred without notice,
          after which this version may no longer be accessible.%
        }%
      }%
    }%
  }%
}

\centering
\includegraphics[width=\textwidth]{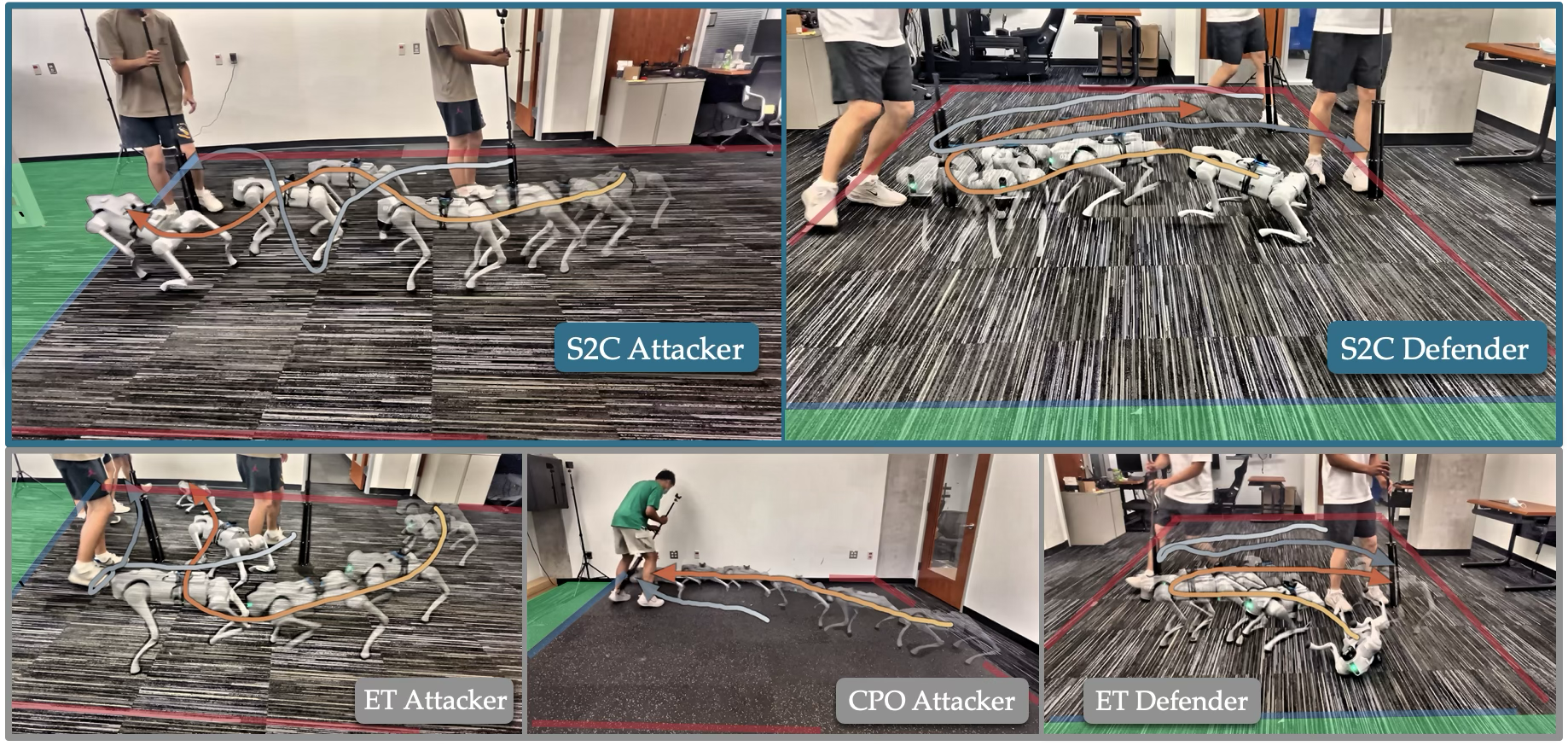}
\captionof{figure}{
Our proposed \textbf{Safety to Competence (S2C)} enforces robot safety without sacrificing competitive performance by training multi-agent reinforcement learning (RL) policies with a pretrained robust safety filter embedded in the environment.
Here, a quadrupedal robot (orange trajectory) plays a touchdown game against 
a human opponent (grey trajectory), where the attacker must reach the green touchdown region before timeout; the robot loses immediately if it falls, crosses the red field boundary, or initiates a collision.
\textbf{Top:} S2C learns effective attacking and defending tactics.
As attacker, it feints past the human to score a touchdown; as defender, it tracks rapid human direction changes and holds the line.
\textbf{Bottom:} safe RL baselines fail frequently under competitive opponent strategies:
the ET attacker exits the field while turning, the CPO attacker collides with the human, and the ET defender falls while pursuing the human.
Project webpage: \textcolor{jhublue}{\url{https://alliance-ai.cs.jhu.edu/s2c/}}.
 }
\label{fig:hardware}
\vspace{0em}
}]
\begingroup
\renewcommand{\thefootnote}{}
\footnotetext{%
$^{1}$R. Wu, R. Yang, and H. Hu are with the Department of Computer
Science, Johns Hopkins University, Baltimore, MD 21218, USA.
{\tt\small \{rwu63,yrui5\}@jh.edu, haimin@cs.jhu.edu}

$^{2}$D. D. Oh and D. P. Nguyen are with the Department of Electrical
and Computer Engineering, Princeton University, Princeton, NJ 08544, USA.
{\tt\small \{do9948,duyn\}@princeton.edu}

$^{*}$Equal contribution.
$^{\dagger}$Equal advising.%
}
\endgroup

\thispagestyle{empty}
\pagestyle{empty}

\begin{abstract}
Robots deployed for competitive tasks must outmaneuver their opponents without sacrificing safety.
Existing approaches, including safe reinforcement learning (RL), train a single policy to achieve task success and avoid failures simultaneously.
This coupling can complicate training and leave the learned policy exploitable by deliberate attacks.
We propose Safety to Competence (S2C), a two-stage RL framework that separates safety synthesis from competitive task learning.
We formulate competitive interactions as safety-critical Markov games and prove that perfect filtering preserves policy non-exploitability when all players commit to safe maneuvers.
S2C learns a robust safety filter via adversarial RL, embeds it in the environment during task policy training, and retains the same filter at deployment.
In simulated touchdown games, S2C outperforms eight safe RL baselines, achieving the highest win rate and Elo rating, and the lowest exploitability.
Hardware stress tests against a human opponent confirm S2C’s competence.
\end{abstract}

\begin{pts}
  \item One-paragraph version of the story, to be turned into the abstract:
  a robot in a competitive interaction is usually made safe by paying for it in
  performance. We give a safety specification and a filter design under which
  the price is provably zero: the filtered game has the same equilibrium and the
  same exploitability as the unfiltered one, and the filtered robot never faults.
  \item Validated in two head-to-head robot scenarios, one asymmetric and one
  symmetric.
\end{pts}

\section{Introduction}
\label{sec:intro}
A wide range of robotic applications involves competitive interactions with other agents in close proximity, ranging from drone racing~\cite{kaufmann2023champion} and motorsports~\cite{wurman2022outracing,oh2025safety} to robot soccer~\cite{haarnoja2024learning} and pursuit--evasion games~\cite{bajcsy2024learning}.
These tasks require strategic play under strict safety constraints, since a safety violation can cause catastrophic hardware damage or bodily harm and, in these games, result in an immediate loss.
Existing approaches for safety-critical, competitive robot motion planning, including a broad class of safe \gls{rl} algorithms~\cite{JMLR:v16:garcia15a}, focus on synthesizing a single policy that achieves safety and task competence simultaneously.
Common techniques include penalizing safety violations in the reward~\cite{Tessler2018RewardCP}, constraining the expected cumulative safety cost~\cite{altman:inria-00074109,Achiam2017ConstrainedPO}, or terminating the episode upon a violation~\cite{Sun2021SafeEB}. 
Coupling the two objectives in this way complicates training and leaves a trade-off between safety and competence that is difficult to tune~\cite{Tessler2018RewardCP}.
Training therefore tends to converge to an overly conservative or overly aggressive strategy.
Both are vulnerable to deliberate attack by an exploitative opponent~(see, for example, Figs.~\ref{fig:hardware} and~\ref{fig:sym-sheet}).

\begin{figure*}[!t]
\centering
\includegraphics[width=\textwidth]{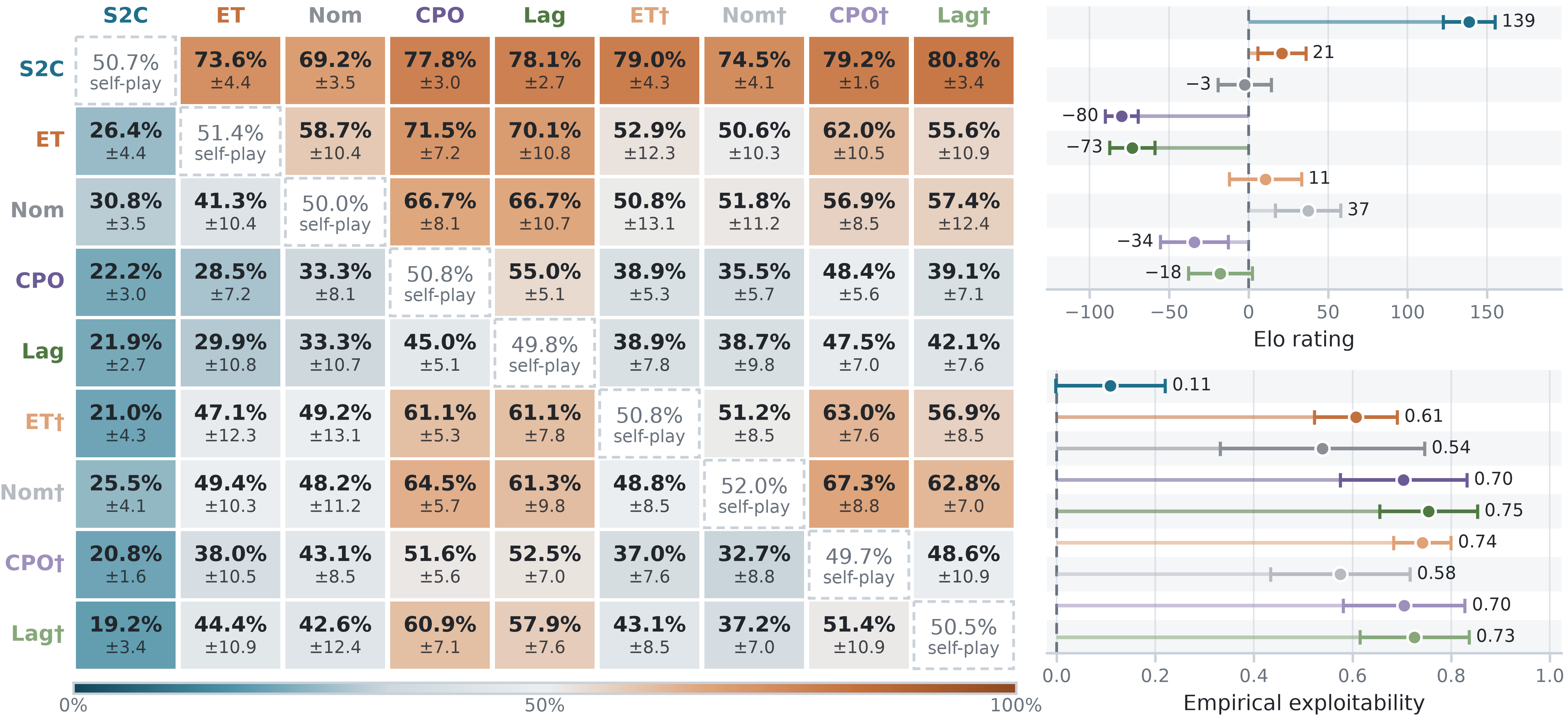}
\caption{Symmetric touchdown game results.
Each cell corresponds to 3900 trials.
S2C outperforms all baselines, achieving the highest Elo rating and lowest empirical exploitability.
\textbf{Left:} round-robin win-rate matrix, where each cell reports the row method's win rate against the column method, averaged over all pairs of training seeds.
\textbf{Top right:} Elo ratings computed jointly based on all matches (higher is better).
\textbf{Bottom right:} empirical exploitability, computed for each seed as the net win rate achieved by its strongest evaluated opponent and then averaged across seeds (lower is better).
} 
\label{fig:sym-sheet}
\end{figure*}

Safety filters~\cite{hsu2023safety} offer a principled approach to separate safety from the task policy.
A safety filter supervises the task policy and modifies its proposed action to prevent an upcoming failure.
Recent work demonstrates that (adversarial) \gls{rl} can be used to synthesize such filters for robots that interact with antagonistic opponents and operate in high-dimensional observation and action spaces~\cite{wang2024magics,nguyenhsu2024gameplay,seo2025uncertainty,Oh2026SynthesisAD}.
Restricting the task policy's actions with a filter is often seen as sacrificing the robot's overall performance. Recent work~\cite{oh2025provably} addresses this concern by proving that for a single agent, perfect (i.e., least restrictive) filtering preserves the asymptotic optimal return in safety-constrained \gls{rl}.
However, this guarantee does not carry over directly to competitive multi-agent settings, where each player's task performance and safety depend on the opponent's strategy.

\para{Contributions}
This paper establishes a game-theoretic foundation for safety-critical, competitive robot interaction and develops a practical algorithmic framework for minimizing policy exploitability without sacrificing safety.
\emph{Our key insight is that separating safety from task objectives in \gls{rl} can substantially reduce policy exploitability}, provided the same safety filter is applied during both training and deployment.
Our key contributions are:
\begin{itemize}
  \item We formulate competitive robot interactions as safety-critical Markov games and show for the first time that perfect safety filters yield non-exploitable Nash equilibria within the class of safe policies.
  \item We propose a two-stage S2C framework that first learns a robust safety filter and then embeds it in the \gls{rl} environment to learn safe and competitive robot policies.
  \item We extensively evaluate S2C in simulated quadruped touchdown games. It achieves the highest win rates and, in the symmetric case, the highest Elo rating and lowest empirical exploitability compared to eight safe \gls{rl} baselines. Hardware tests further demonstrate S2C's robustness to human strategies unseen during training.
\end{itemize}

\begin{pts}
  \item Hook: competitive robot interactions (racing, contested navigation,
  sports, dense merging) require both hard safety and competitive strength.
  \item Gap: current practice buys safety with competence. Penalty shaping,
  \gls{cmdp} budgets and conservative filters all shrink what the robot may do,
  and none of them can say by how much the robot's competitive standing suffered.
  \item Insight: whether the price is zero is a property of the safety
  \emph{specification}, not of the filter. Fault-attributed specifications make
  the interaction zero-sum and unilaterally enforceable; then a least-restrictive
  filter is free.
  \item Contributions (3--4 bullets, mirroring Secs.~\ref{sec:formulation}--\ref{sec:experiments}).
\end{pts}

\section{Related Work}
\label{sec:related}
\para{Safety Filters}
A safety filter supervises a task controller and modifies its proposed action only when necessary to avoid catastrophic failures.
Existing filter designs include, for example, model predictive safety filters (MPSFs)~\cite{wabersich2022predictivecontrolbarrierfunctions}, \glspl{cbf}~\cite{ames2019control, robey2020learning}, and \gls{hj} reachability~\cite{bansal2017hamilton}.
Applying these methods to high-dimensional systems with black-box dynamics presents modeling and computational challenges. MPSFs require a predictive dynamics model and online trajectory optimization, which can be computationally expensive.
Conventional model-based \gls{cbf} filters require known dynamics information; constructing such a certificate is challenging for nonlinear dynamics and calls for a case-by-case analysis.
Classical grid-based reachability methods scale poorly with the state dimension.
Neural approximations have recently enabled filter synthesis for higher-dimensional robotic systems~\cite{fisac2019bridging,bansal2021deepreach}.
To account for disturbances that deliberately seek to cause failures, adversarial safety \gls{rl} methods train a disturbance policy against the controller, jointly learning a safety critic and a fallback policy~\cite{hsu2023isaacs,wang2024magics}.
Oh et al.~\cite{Oh2026SynthesisAD} use this approach to synthesize a robust \gls{qcbf}, enabling model-free, smooth safety filtering. We adopt this technique in our S2C framework to approximate perfect filters and achieve task-efficient safety overrides.

\para{Safe Reinforcement Learning}
Existing safe \gls{rl} methods predominantly incorporate safety into policy learning through reward penalties or cost constraints~\cite{JMLR:v16:garcia15a}.  A widely adopted safe \gls{rl} formulation is the \gls{cmdp}, which maximizes the expected task payoff subject to a budget on the expected cumulative safety cost. \gls{cpo}~\cite{Achiam2017ConstrainedPO}
approximately enforces this budget during policy updates, while Lagrangian methods~\cite{Tessler2018RewardCP} convert the constraint into an adaptive penalty. Early termination instead ends the episode at a failure event, but does not actively prevent the policy from reaching that condition. 
These approaches do not generally guarantee safety in competitive interactions. A positive expected-cost budget can still permit rare, yet catastrophic violations, and reward penalties may favor higher task return at the expense of safety on some trajectories. Moreover, the trajectory distribution depends on the opponent's strategy; a policy that satisfies its cost budget against training opponents may exceed it against an unseen opponent at deployment.
In this paper, we advocate for a safety--task separation principle that first trains a robust safety filter and then embeds it in task-policy training.
We prove that, under perfect filtering, Nash equilibria of the filtered game induce non-exploitable policies. Empirically, we demonstrate that this approach achieves substantially stronger performance and lower exploitability than the above safe \gls{rl} baselines.

\para{Game-Theoretic Multi-Robot Interaction}
Dynamic games model robot interactions where each
agent's outcome depends on both its own policy and others' responses~\cite{basar1998dynamic}.
A popular model-based game solver, the iterative linear-quadratic game~\cite{fridovich2020efficient}, computes local feedback Nash equilibria by linearizing the dynamics and quadratizing the players' objectives.
Lidard et al.~\cite{lidard2024blending} extend this approach to incorporate a data-driven reference policy.
These approaches generally require known dynamics and differentiable objectives, and their runtime computational cost can limit their applicability to high-dimensional robotic systems.
For complex or unknown dynamics, \gls{marl} instead learns strategic policies through simulated interactions. Centralized-training methods use joint information during
training to mitigate the nonstationarity caused by simultaneously learning
agents~\cite{Lowe2017MultiAgentAF}, while self-play and population-based methods train against an evolving set of opponents~\cite{Lanctot2017AUG}. \Gls{ippo}~\cite{Witt2020IsIL} provides a decentralized approach by optimizing policies for each agent using its own
observations without requiring a centralized critic or an explicit opponent model.
We use \gls{ippo} as the task-policy training method in the filtered environment.

\begin{pts}
  \item Safety filters and \gls{hj} reachability~\cite{bansal2017hamilton};
  \glspl{cbf}~\cite{ames2019control}; predictive filters.
  \item Safe \gls{rl}~\cite{garcia2015comprehensive}: \gls{cmdp}~\cite{altman1999constrained}
  and Lagrangian methods, and why an
  expected-cost budget cannot give a per-player guarantee in a game.
  \item Game-theoretic multi-robot interaction: equilibrium solvers,
  \gls{marl} self-play, exploitability as a quality measure.
  \item Fault attribution and right-of-way rules in autonomous driving:
  \gls{rss}~\cite{shalevshwartz2017formal} and its successors fix blame by a
  hand-written rule set over closed-form kinematics; here the same attribution
  principle is enforced by a learned robust certificate and carries an
  equilibrium guarantee.
\end{pts}

\section{Problem Formulation}
\label{sec:formulation}

We consider competitive, safety-critical interactions between two robots (or two groups of robots), in which each player seeks to maximize its task performance over time against its opponent; a player \textit{loses immediately} if it violates a safety constraint.
Formally, we model such interactions with the following Markov game.

\begin{definition}[Safety-Critical Markov Game]
\label{def:scmg}
A \gls{scmg} is defined as the tuple
\begin{equation}
\label{eq:scmg-tuple}
\begin{split}
\gameSC := \big(
\statespace,\actionSet^1,\actionSet^2,
\transkernel,\discount;
\reward,\rho;
\consFunc^1,\consFunc^2;
\consFunc^{1,\mathrm{c}},
\consFunc^{2,\mathrm{c}};
\attrib
\big),
\end{split}
\end{equation}
where $\state=(\state^1,\state^2)\in\statespace$ is the physical joint state,
and $\actionSet^\iagent$ is the action space for player $\iagent\in\{1,2\}$,
$\transkernel(\cdot\mid\state,\action^1,\action^2)$ is the transition kernel,
$\discount\in(0,1)$ is a discount factor,
task reward
$\reward:\statespace\rightarrow\reals$
is bounded,
parameter $\rho>0$ specifies the absorbing win--loss reward, and $\attrib$ is the \textit{adjudication map}, which we explain next.

\noindent\emph{Safety Decomposition.}
For each player $\iagent$, safety is specified by a \emph{solo} margin function $\consFunc^\iagent(\state^\iagent)$
and a \emph{coupled} margin $\consFunc^{\iagent,\mathrm{c}}(\state^1,\state^2)$, where negative values indicate violations.
The total safety margin is
$\consTot^\iagent(\state)
:=
\min\left\{
\consFunc^\iagent(\state^\iagent),
\consFunc^{\iagent,\mathrm{c}}(\state^1,\state^2)
\right\},
$
and the corresponding \textit{failure set} is
$
\failset^\iagent
:=
\left\{
\state\in\statespace
\;\middle|\;
\consTot^\iagent(\state)<0
\right\}.
$
If $\state$ enters $\failset^1\cup\failset^2$, an \textit{adjudicated outcome} is triggered and the game is terminated.
\end{definition}

\para{Fault Attribution}
In an \gls{scmg}, the winning condition depends on fault attribution.
Conceptually, a player $\iagent$ wins if the opposing party violates safety.
However, it is possible that both players' safety specifications are violated simultaneously, i.e., $\state \in \failset^1\cap\failset^2$.
Below, we specify how such states are adjudicated.

\begin{assumption}[Unique Fault Attribution]
\label{assmp:ufa}
The adjudication map $\attrib: \statespace \rightarrow \{1,2\}$ identifies the \textit{unique} player held at fault at each state $\state \in \failset^1\cap\failset^2$, and $\attrib(\state)=\iagent$ whenever $\state \in \failset^\iagent\setminus\failset^{-\iagent}$.
\end{assumption}

\begin{example}
Throughout the paper, we use a \textit{symmetric} two-player quadruped touchdown game as a running example.
The game is played on a $5.2\times3.0$ m field, with two Unitree Go2 quadrupedal robots initialized at opposite ends and competing to reach the opponent's touchdown line first.
A player receives a (sparse) reward $\reward = +1$ when it first reaches its specified touchdown line and $-1$ when the opponent does.
The solo safety margin $\consFunc^\iagent$ is defined as the minimum distance to the ground and the field boundary, and the coupled margin $\consFunc^{\iagent,\mathrm{c}}$ measures the relative distance to the opponent.
The adjudication map $\attrib$ is defined such that, in the event of an inter-robot collision, the robot with the higher velocity projected toward the opponent (i.e., one that closes faster on its opponent) is attributed fault.
\end{example}

In order to formally show that \gls{scmg}~\eqref{eq:scmg-tuple} is zero-sum, we consider an equivalent state space $\bar{\statespace}:=
\big(\statespace\setminus(\failset^1\cup\failset^2)\big)\cup\{\top,\bot\}$, where $\top$ and $\bot$ are \textit{absorbing states} denoting an immediate win and loss for Player 1.
Subsequently, we consider the following augmented reward:
\begin{equation}
\label{eq:augmented-reward}
\bar{\reward}(\state) := \begin{cases} +\rho, & \state=\top, \\ -\rho, & \state=\bot, \\ \reward(\state), & \state\in\statespace\setminus(\failset^1\cup\failset^2). \end{cases}
\end{equation}
At each time, player~$\iagent$ receives stage reward $\turn^\iagent\bar{\reward}$, where $\turn^1:=+1$ and $\turn^2:=-1$.
Consequently, along every realized trajectory,
$
\sum_{\tdisc=0}^{\infty}\discount^\tdisc\big(\turn^1\bar{\reward}(\state_\tdisc)+\turn^2\bar{\reward}(\state_\tdisc)\big)=0
$.

\begin{proposition}
\label{prop:zero-sum}
Under Assumption~\ref{assmp:ufa}, \gls{scmg}~\eqref{eq:scmg-tuple} is zero-sum.
\end{proposition}

\para{Exploitability and Game Value}
To quantify a policy's~\textit{robust competitiveness} in this zero-sum game, we next establish the game's \textit{value} and introduce \emph{exploitability}, which measures how vulnerable a robot is to a best-responding, exploitative opponent that deliberately targets its weaknesses.

Recall that the \gls{scmg}~\eqref{eq:scmg-tuple} is a discounted zero-sum stochastic game with bounded stage rewards.
Under the standard minimax assumptions~\cite{shapley1953stochastic}, the discounted stochastic game guarantees that its max--min and min--max payoffs coincide.
We therefore obtain the following proposition.

\begin{proposition}[Existence of the Game Value]
\label{prop:game-value}
The \gls{scmg}~\eqref{eq:scmg-tuple} admits a value at every $\state\in\bar{\statespace}$:
\begin{equation}
\label{eq:game-value}
\begin{aligned}
\valFunc^*(\state)
&=\sup_{\policy^1\in\policyAll{1}}\inf_{\policy^2\in\policyAll{2}}\outcome(\policy^1,\policy^2)(\state)\\
&=\inf_{\policy^2\in\policyAll{2}}\sup_{\policy^1\in\policyAll{1}}\outcome(\policy^1,\policy^2)(\state),
\end{aligned}
\end{equation}
where $\outcome(\policy^1,\policy^2)(\state):=\expect\!\left[\sum_{\tdisc=0}^{\infty}\discount^\tdisc\bar{\reward}(\state_\tdisc)\,\middle|\,\state_0=\state\right]$.
\end{proposition}

Proposition~\ref{prop:game-value} implies that there exists a stationary \emph{\gls{NE}} policy pair $(\policy^{1,*},\policy^{2,*})$ that guarantees the value $\valFunc^*$ against every opponent policy.
In this case, an \gls{NE} policy is \textit{non-exploitable}: neither player can improve its expected payoff by changing to another policy~\cite[Ch.~2]{basar1998dynamic}.

\begin{definition}[Exploitability]
\label{def:expl}
For a policy $\policy^\iagent$, its \textit{exploitability} is defined as the gap between the game value and the worst-case payoff:
\begin{equation}
\label{eq:expl}
\expl(\policy^\iagent)(\state):=\turn^\iagent\valFunc^*(\state)-\inf_{\policy^{-\iagent}\in\policyAll{-\iagent}}\turn^\iagent\outcome(\policy^1,\policy^2)(\state),
\end{equation}
where $\turn^1=+1$ and $\turn^2=-1$.
\end{definition}

At an \gls{NE} $(\policy^{1,*},\policy^{2,*})$ of \gls{scmg}, neither player's strategy can be exploited, i.e., $\expl(\policy^{1,*})(\state)=\expl(\policy^{2,*})(\state)=0$.

\begin{figure*}[t]
\centering
\includegraphics[width=\textwidth]{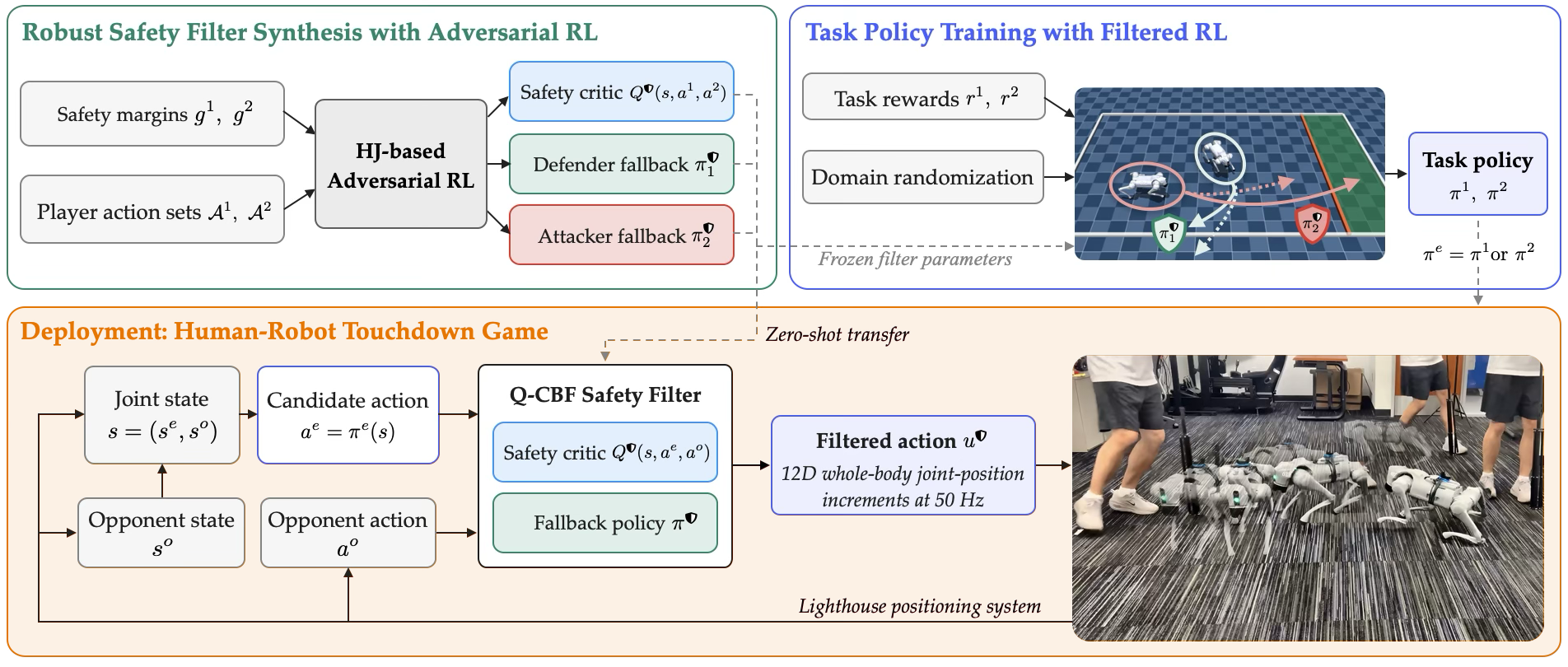}
\caption{Overview of the proposed \textbf{S2C} training and deployment pipeline, illustrated with the asymmetric touchdown game.}
\label{fig:system-overview}
\end{figure*}

\section{Analysis: Non-Exploitability under Perfect Safety Filtering}
\label{sec:main-result}

In this paper, our goal is to improve robot competitiveness with safe yet minimally exploitable policies.
However, direct equilibrium learning in an \gls{scmg} is challenging since it couples task optimization with safety reasoning; each player needs to preempt safety risks induced by close-call interactions with the opponent while searching for winning strategies.
To this end, we propose a two-stage policy synthesis framework that \textit{separates safety enforcement from competitive task objectives} (Fig.~\ref{fig:system-overview}).
First, we synthesize a robust safety filter~\cite{hsu2023safety} for each player using \gls{hj}-based adversarial \gls{rl}.
Second, we train task policies in an \gls{rl} environment that incorporates the filters into its transition dynamics, and deploy each policy with the same filter.
Because the same filters are used during both training and deployment, the task policy adapts to the filtered game dynamics without requiring an explicit filter model or additional reward terms.

In this section, we establish the theoretical foundation for safety--task separation in zero-sum, safety-critical games.
Our main result shows that, under perfect filtering, an \gls{NE} of the filtered game induces an \gls{NE} of the original \gls{scmg} in the space of safe policies. Consequently, equilibrium policies executed under the same filters remain safe and non-exploitable against safety-aware opponents.
We then introduce a practical instantiation of this framework in Sec.~\ref{sec:method}.

Let $\safeSet^{\iagent*}\subseteq\statespace\setminus\failset^\iagent$ denote the \textit{maximal robust controlled-invariant set}~\cite{blanchini1999set} for player~$\iagent$.
Inside $\safeSet^{\iagent*}$, player $\iagent$ is guaranteed a strategy to stay within $\safeSet^{\iagent*}$ regardless of the opponent's actions.
Next, we define the \textit{safe action set} as
\begin{equation*}
\actionSafe{\iagent}(\state):=\left\{\action^\iagent\mid\transkernel(\safeSet^{\iagent*}\mid\state,\action^\iagent,\action^{-\iagent})=1,\ \forall\action^{-\iagent}\in\actionSet^{-\iagent}\right\}.
\end{equation*}

\begin{definition}[Perfect Safety Filter~\cite{hsu2023safety}]
\label{def:filter}
A (robust) safety filter $\safetyFilter^\iagent:\statespace\times\actionSet^\iagent\rightarrow\actionSet^\iagent$ is \emph{perfect} if it satisfies
$
\safetyFilter^\iagent(\state,\action)\in\actionSafe{\iagent}(\state),~\forall\action\in\actionSet^\iagent,~\forall \state\in\safeSet^i,
$
and leaves every safe action unchanged:
$
\safetyFilter^\iagent(\state,\action)=\action,~ \forall\action\in\actionSafe{\iagent}(\state),\forall \state\in\safeSet^i
$.
\end{definition}

\begin{definition}[Filtered Game]
\label{def:filtered-game}
Given perfect filters $\safetyFilter=(\safetyFilter^1,\safetyFilter^2)$, we define a \textit{filtered game} as
\begin{equation}
\label{eq:filtered-game}
\gameFilt:=(\bar{\statespace},\actionSet^1,\actionSet^2,\transkernel_{\safetyFilter},\bar{\reward},\discount),
\end{equation}
where $\transkernel_{\safetyFilter}(\cdot\mid\state,\action^1,\action^2):=\bar{\transkernel}(\cdot\mid\state,\safetyFilter^1(\state,\action^1),\safetyFilter^2(\state,\action^2))$.
\end{definition}

Since perfect filters leave safe actions unchanged, any payoff improvement a player can achieve by deviating to another safe policy is also achievable in the filtered game.
This property allows us to compare the equilibria of a filtered game $\gameFilt$ with those of $\gameSC^{\mathrm{safe}}$, the original \gls{scmg} with each player committed to a safe strategy.

\begin{theorem}[Non-Exploitability under Perfect Filtering]
\label{thm:main}
Let $\safetyFilter=(\safetyFilter^1,\safetyFilter^2)$ be perfect filters, and let $\gameSC^{\mathrm{safe}}$ denote the original \gls{scmg} with each player restricted to safe policy classes $\policySafe{\iagent}:=\left\{\policy^\iagent\in\policyAll{\iagent}\;\middle|\;\policy^\iagent\big(\actionSafe{\iagent}(\state)\mid\state\big)=1,\ \forall\,\state\in\safeSet^{\iagent*}\right\}$.
For any initial state $\state\in\safeSetMax$, an NE $(\policy^{1,*},\policy^{2,*})$ of $\gameFilt$ induces, through the same filters, an executed pair $(\policyExecd{1},\policyExecd{2})$ that is an NE of $\gameSC^{\mathrm{safe}}$. 
Consequently, both players are non-exploitable in their respective games:
\begin{equation}
\label{eq:transfer}
\expl_{\gameSC^{\mathrm{safe}}}(\policyExecd{\iagent})(\state)=\expl_{\gameFilt}(\policy^{\iagent,*})(\state)=0.
\end{equation}
\end{theorem}

\begin{proof}
Let $(\policy^{1,*},\policy^{2,*})$ be an \gls{NE} of $\gameFilt$. Suppose a player could improve its payoff by a safe unilateral deviation from the executed pair $(\policyExecd{1},\policyExecd{2})$ in $\gameSC^{\mathrm{safe}}$. Since a perfect filter leaves safe actions unchanged, the deviating policy could be used as a nominal task policy in $\gameFilt$, which achieves the same improvement in $\gameFilt$, contradicting its \gls{NE} property.
Hence, the executed policy pair is an \gls{NE} of $\gameSC^{\mathrm{safe}}$.
The equilibrium payoffs coincide, and Definition~\ref{def:expl} ensures zero exploitability for both players in their respective games.
\end{proof}

Theorem~\ref{thm:main} shows that in competitive games, the \textit{safety--task separation} does not necessarily sacrifice exploitability when both players are safety-aware. Based on this insight, Sec.~\ref{sec:method} presents an \gls{rl} framework for synthesizing competitive robot policies using learned robust safety filters.

\section{Method: Learning Minimally Exploitable Robot Policies with Filtered RL}
\label{sec:method}

In this section, 
we present our main algorithmic contribution, \textbf{Safety to Competence (S2C)}, a two-stage \gls{rl} framework that trains safe and minimally exploitable robot policies.
It sequentially trains a robust safety filter and a task policy with the filter embedded in the environment.
Our overall system architecture and algorithmic approach are summarized
in Fig.~\ref{fig:system-overview} and Algorithm~\ref{alg:s2c}. 
Our framework is modular and agnostic to the specific methods used for filter synthesis, task policy training, and filter deployment.
In this paper, we instantiate filter synthesis, task-policy training, and filter deployment with MAGICS~\cite{wang2024magics},   \gls{ippo}~\cite{Witt2020IsIL}, and Robust \gls{qcbf}~\cite{Oh2026SynthesisAD,oh2025safety}, respectively.

\begin{pts}
  \item Unnumbered opening paragraph. The method is \emph{two training runs}
  and one runtime composition. The first synthesizes a safety certificate from
  the margins alone; the second trains the task policy on $\reward$ alone,
  inside the game that certificate has already closed off. Neither run is given
  the other's objective, and the task policy is never told that a filter exists.
  \item Framing sentence to carry the reader across from
  Sec.~\ref{sec:formulation}: the guarantee assumes a perfect filter and an
  exact equilibrium; Sec.~\ref{sec:method-filter} replaces the first and
  Sec.~\ref{sec:method-task} the second, and neither replacement is free.
\end{pts}

\begin{algorithm}[t]
\caption{S2C: Safety-Filtered Competitive \gls{rl}}
\label{alg:s2c}
\begin{algorithmic}[1]
\Require margins $\consFunc^\iagent,\consFunc^{\iagent,\mathrm{c}}$, task reward $\reward$, 
initial task policy~$\policy^i_{\rm{init}}$, 
robust filter synthesis algorithm \textsc{Filter\_Synthesis},
multi-agent reinforcement learning algorithm \textsc{MARL}
\vspace{3pt}
\Statex \textit{Training Stage 1: Robust Filter Synthesis}
\vspace{1pt}
\State $(\qFunc^\shield,\fallback_1,\fallback_2)\gets$ \textsc{Filter\_Synthesis}($\consFunc^1,\consFunc^{1,\mathrm{c}},\consFunc^2,\consFunc^{2,\mathrm{c}}$)
\State \textproc{Robust\_Filter} $\gets (\qFunc^\shield,\fallback_1,\fallback_2)$
\vspace{3pt}
\Statex \textit{Training Stage 2: Filtered Multi-Agent \gls{rl}}
\State Initialize policy pool: $\oppPool\gets\emptyset$
\State Initialize task policies: $(\policy^1,\policy^2)\gets(\policy^1_{\rm{init}},\policy^2_{\rm{init}})$
\For{$\iter=1,2,\dots$}
  \State Collect rollouts with \textproc{Robust\_Filter}; each player faces its opponent using $\policy^{-\iagent}$ or, with probability $\poolProb$, a policy in $\oppPool$
  \State Train $\policy^1,\policy^2$ with \textsc{MARL}
  \State Periodically update $\oppPool\gets\oppPool\cup\{\policy^1,\policy^2\}$
\EndFor
\Statex \textit{Deployment}
\State $\policy^\ego \gets \policy^1$ or $\policy^2$
\State $\action^\ego_\tdisc\gets$ \textproc{Robust\_Filter}($\state_\tdisc,\policy^\ego$)
\end{algorithmic}
\end{algorithm}

\subsection{Synthesizing Robust Safety Filters}
\label{sec:method-filter}

For each player $\iagent$, we first synthesize its robust safety filter independently of the task objective.
We formulate safety as a zero-sum game between the robot and a worst-case opponent: the ego robot seeks to maintain its own safety, while the adversary actively drives it toward failure.
This yields a learned \textit{safety fallback policy} $\policy^\shield_\iagent$ trained to protect player~$\iagent$ against adversarial opponent behavior.

To solve this safety game for general, high-dimensional interactive tasks (e.g., the touchdown game in the Running Example with 62D observation and 12D action spaces), we use MAGICS~\cite{wang2024magics}, a game-theoretic adversarial \gls{rl} algorithm for neural safety synthesis.
MAGICS jointly trains a safety fallback $\policy^\shield_\iagent$, a worst-case adversarial policy, and a safety critic $\qFunc^\shield$.
It is shown to converge to a local minimax solution, providing a scalable approximation to \gls{hj}-based robust safety analysis that would yield a perfect safety filter~\cite{oh2025provably}.
The fallback $\policy^\shield_\iagent$ and critic $\qFunc^\shield$ form essential parts of the robust safety filter $\safetyFilter^\iagent(\state,\policy^\iagent)$, which takes a state $\state$ and a task policy $\policy^\iagent$ as inputs, simulates the worst-case opponent action, and returns a filtered task action $\action^\iagent$.
We explain the deployment of this robust filter in Sec.~\ref{sec:method-deploy}.

\begin{example}
In the touchdown game, both the safety fallback and critic operate with a 62D observation comprising both ego and opponent's positions of the body frame, body velocities and orientations, axial rotational
rates, and angular velocity.
The robot action $\action^\iagent\in\actionSet^\iagent=[-1,1]^{12}$ is a normalized joint-position increment command for the 12 actuated leg joints, tracked by a joint-level PD controller.
Both the safety fallback and the critic are three-layer MLPs, with hidden layer widths of 512 and 256 units, respectively, both trained with 2048 parallel environments in MuJoCo~\cite{todorov2012mujoco}
on a workstation with an RTX 5080 GPU.
\end{example}

\subsection{Learning Task Policies in the Filtered Environment}
\label{sec:method-task}
We train each agent's \emph{task policy} $\policy^\iagent$ in an \gls{rl} environment that embeds its safety filter $\safetyFilter^\iagent$, as in the filtered game of Definition~\ref{def:filtered-game}.
To train both agents' task policies simultaneously, our framework calls for a \gls{marl} algorithm, e.g., \gls{ippo}~\cite{Witt2020IsIL} or MAPPO~\cite{yu2022surprising}.
At each training step, each player's task policy $\policy^\iagent$ proposes a candidate action $\action^\iagent$, and the simulator executes the action modified by the safety filter.
Crucially, the task policies are learned based on the agents' interaction data with the environment, requiring neither explicit information on the safety filter nor extra reward terms to account for safety overrides.

To stabilize training and prevent overfitting to a specific opponent strategy, we maintain a \textit{policy pool} $\oppPool$ of historical policy checkpoints~\cite{heinrich2015fictitious, Lanctot2017AUG} and sample opponents from this pool in a fraction of parallel environments.
We prioritize opponents with a win rate close to $50\%$ to focus learning on balanced matches.

\begin{example}
The task reward consists of a sparse $\pm1$ touchdown reward 
and a dense locomotion shaping term that is annealed after 1000 iterations.
We train each player's task policy $\policy^\iagent$ with \gls{ippo}~\cite{Witt2020IsIL}.
The task policy uses the same observation space as the safety filter
and outputs 12D joint-position targets, which are converted into joint-position increments prior to filtering.
Each player's policy is parameterized by an MLP with three hidden layers of 512, 256, and 128 neurons.
Training is conducted across 8192 parallel environments with an opponent-pool sampling probability of $\poolProb=0.35$, running for 5000 iterations.
\end{example}

\begin{pts}
  \item Replaces the \emph{perfect filter} idealization. One adversarial
  \gls{rl} run on the joint state (\gls{isaacs}-style
  actor--critic~\cite{hsu2023isaacs}) returns three networks at once: a safety
  critic, an opponent actor conditioned on the ego action, and the fallback
  actor $\fallback$. There is no second stage that turns a value function into
  a filter -- certificate and fallback come out of the same run.
  \item Composing the critic with the opponent actor gives a robust action value
  $\wh\qFunc(\state,\action)$; evaluating it at the fallback gives
  $\wh\valFunc(\state)$. The opponent is adversarial, so what the certificate
  certifies holds whatever the other robot does.
  \item Convergence is decided by a monotone ratchet on the certified set; no
  hand-set threshold enters.
\end{pts}

\subsection{Deploying Task Policies with Robust Safety Filters}
\label{sec:method-deploy}

At deployment, each learned task policy $\policy^\iagent$ is paired with the same robust safety filter $\safetyFilter^\iagent$ used during training.
The filter is constructed from the safety critic $\qFunc^\shield$ and fallback policy $\policy^\shield$ learned in Sec.~\ref{sec:method-filter}.
At every control step, the task policy proposes a candidate action, which is passed through unchanged when it satisfies the learned safety constraint and otherwise corrected by the filter.
Using the same filter during training and deployment is critical to our framework.
During training, the task policy experiences the consequences of filter interventions directly through the filtered environment and learns to \textit{cooperate with the filter} to minimize exploitability (Theorem~\ref{thm:main}) by applying task strategies that avoid unnecessary overrides.

\begin{example}
For the touchdown game, we approximate the perfect filter in Definition~\ref{def:filter} using a \gls{qcbf}, which has been shown to robustly maintain safety while preserving task performance in quadrupedal locomotion~\cite{Oh2026SynthesisAD}, with the learned critic $\qFunc^\shield$ and fallback $\policy^\shield_\iagent$.
In particular, the filter performs a fixed number of gradient-descent steps to search for an action closest to the proposed task action and certified safe by $\qFunc^\shield$.
If no such action is found, the filter uses the learned fallback policy $\fallback_\iagent$.
\end{example}

\section{Experiments}
\label{sec:experiments}

We use our proposed {\StoC} framework to synthesize the safety filter and train touchdown game policies in both symmetric and asymmetric scenarios.
We evaluate their safety and competence at scale in simulation, and deploy the robot in both roles (attacker and defender) against a human opponent in a real-world experiment.

\para{Hypotheses}
Our experiments test two hypotheses that reflect the key advantages of \StoC:
\begin{itemize}
    \item \textbf{H1 (Safe competence).} \emph{\StoC{} yields more competitive policies without sacrificing safety.
}

  \item \textbf{H2 (Robustness and generalization).} \emph{\StoC{} generalizes better to out-of-distribution opponents.}
\end{itemize}

\para{Baselines}
We compare \StoC{} to four baseline safe \gls{rl} methods that jointly optimize safety and task objectives.
\begin{itemize}
  \item \textbf{Early termination (\ET)} ends the current training episode on a safety violation.
  \item \textbf{Reward penalty (\nom)} adds a penalty to the reward for each safety violation.
   \item \textbf{Constrained policy optimization (\CPO)} enforces a violation budget $\costBudget^\iagent$ during a trust-region update.
  \item \textbf{PPO-Lagrangian (\Lag)} enforces a violation budget through an adaptive multiplier.
  \item \textbf{Filtered baselines ($\dagger$)} equips each baseline above with the same filter used by \StoC{} at deployment.
\end{itemize}
All baselines are trained with the same \gls{ippo} algorithm, network architecture,
and training schedule as \StoC.
We also swept additional parameters for the \CPO~and \Lag~baselines.

\para{Metrics}
We select checkpoints that achieve the best self-play payoff for all nine methods and use them to form the evaluation pool $\evalPool$.
We evaluate robot performance with the following metrics:
\begin{itemize}
  \item \textbf{Safe rate (SR).} The fraction of episodes free of a safety violation indicated by $\consTot^i(s) < 0$.
  \item \textbf{Win rate (WR).} The averaged share of won episodes against opponent methods.
  \item \textbf{Empirical exploitability.} 
  We measure empirical exploitability for the symmetric game ($\valFunc^{*}=0$).
  We then evaluate ~\eqref{eq:expl} on the empirical games over $\evalPool$~\cite{Lanctot2017AUG} as $\max_{\policy'\in\evalPool}\wh\outcome(\policy',\policy)$, where
  $\wh\outcome(\policy,\policy')=\big(\numEp_{\mathrm{win}}(\policy)-\numEp_{\mathrm{win}}(\policy')\big)/\numEp$ is the net
  win rate over $\numEp$ episodes.
  \item \textbf{Elo.} We calculate Elo ratings for symmetric games based on all pool matches using the standard scale factor of 400, with bootstrap confidence intervals.
\end{itemize}
Simulation values are mean $\pm$ one standard deviation across the five random training seeds of each method.

\subsection{Simulated Touchdown Games}
\label{sec:sim-results}

\begin{figure}[t]
\centering
\includegraphics[width=\columnwidth]{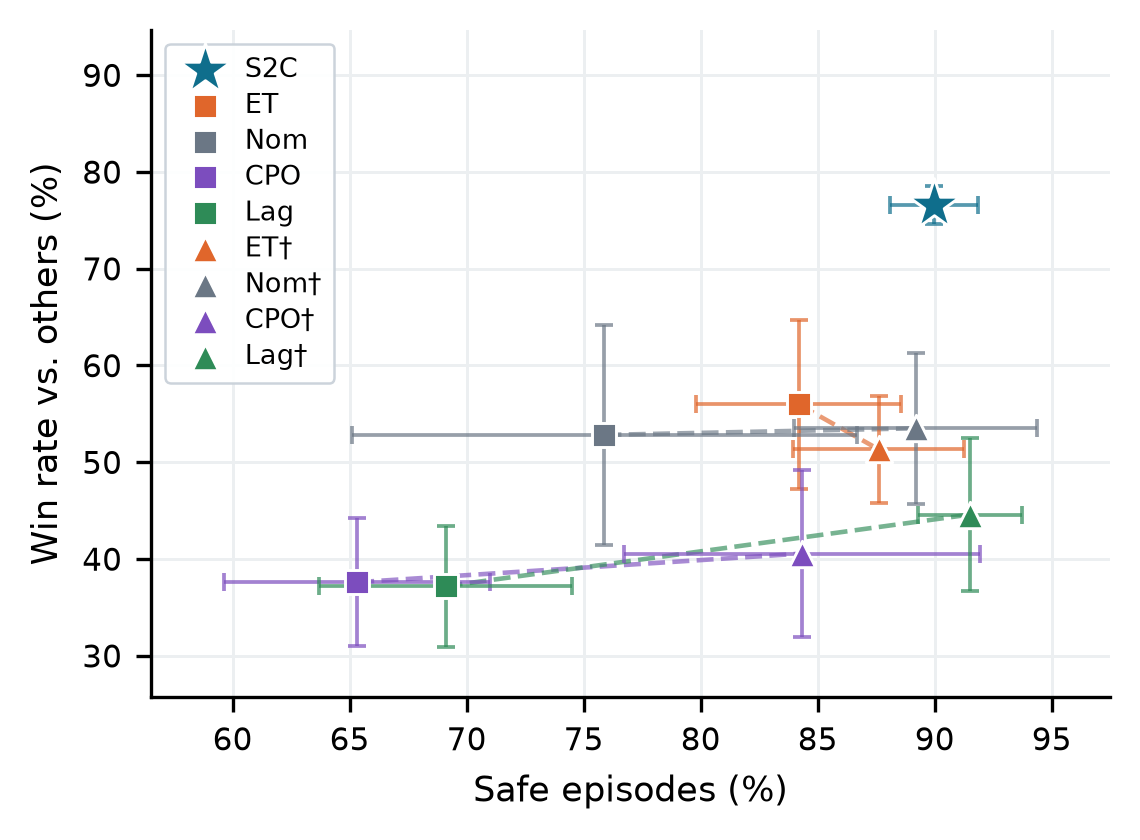}
\caption{Safety versus win rate in symmetric touchdown games.}
\label{fig:pareto-sym}
\end{figure}

\para{Results}
We report the results of the simulated symmetric game (Running Example) in this section. Following the setup in Sections~\ref{sec:method-filter} and~\ref{sec:method-task}, we evaluate all nine methods in a round-robin tournament.
As shown in Fig.~\ref{fig:sym-sheet} and Table~\ref{tab:main}, \StoC{} achieves the highest win rate across all pairwise matches with an average of $76.5\%$. 
\StoC{} also achieves the highest Elo rating ($139$), significantly exceeding the second-place competitor, \nomsf{} ($37$) by more than $100$ points.
In terms of empirical exploitability, the strongest pool opponent achieves a net win rate of $0.11$ against \StoC, compared to $0.54$--$0.75$ against the baselines. 
Because the evaluation pool includes unfiltered baselines, this indicates that \StoC{} remains hard to exploit even by reckless opponents not committed to using safe policies. 
Moreover, \StoC{} achieves a high safe rate of $90.0\%$, close to that of \Lagsf{}.
These results validate \textbf{H1}.

\subpara{Filtered ablation}
We conduct an ablation study by deploying the learned safety filter in \StoC{} for the baselines at runtime (see Figs.~\ref{fig:sym-sheet} and~\ref{fig:pareto-sym}).
The filter enhances safety as expected: the safe rates of filtered baselines improve across the board by up to $22.4$~percentage points (pp), bringing both \Lagsf{} and \nomsf{} to safety levels comparable to \StoC{}.
However, this added safety does not transfer to a notable gain in the robot's competitive competence: baseline win rates change by at most $7.4$\,pp, remaining substantially below \StoC{}, and their exploitability stays high.
This is because the baseline policies were trained in unfiltered environments, and therefore did not learn to cooperate with the safety filter.
In some states, they repeatedly propose unsafe actions, so the filter intervenes frequently and thereby degrades the task performance.
This ablation shows that a safety filter alone does not grant strategic competence; rather, competence must be co-learned within the filtered environment as in \StoC{}.

\subpara{Qualitative observations}
During training, we observe that competitive tactics such as feints and abrupt stop-and-turn maneuvers (see Fig.~\ref{fig:hardware}) emerge noticeably earlier in \StoC{} than in the baseline methods. 
We attribute this to the safety filter intervening to brake or stabilize the robot whenever it approaches danger.
It spares the policy from having to discover constraint-satisfying strategies through trial-and-error failures, allowing it to focus its exploratory effort more directly on task objectives.
In contrast, because the baselines must map safety boundaries through catastrophic failures, individual runs often collapse into overly conservative play or stereotyped gaits that specific opponents can easily exploit during evaluation.

\begin{table}[t]
\centering
\footnotesize
\setlength{\tabcolsep}{5pt}
\caption{Round-robin results in simulated touchdown games.}
\label{tab:main}
\sbox0{\resizebox{\columnwidth}{!}{%
\begin{tabular}{l cc cc}
\toprule
 & \multicolumn{2}{c}{Asymmetric game} & \multicolumn{2}{c}{Symmetric game} \\
\cmidrule(lr){2-3} \cmidrule(lr){4-5}
Method & WR ($\%$) $\uparrow$  & SR ($\%$) $\uparrow$  & WR ($\%$) $\uparrow$  & SR ($\%$) $\uparrow$ \\
\midrule
\StoC & \textbf{77.5}\,$\pm$\,2.8 & \textbf{87.9}\,$\pm$\,2.4 & \textbf{76.5}\,$\pm$\,2.0 & 90.0\,$\pm$\,1.9 \\
\ET & 68.1\,$\pm$\,12.5 & 69.7\,$\pm$\,18.5 & 56.0\,$\pm$\,8.7 & 84.2\,$\pm$\,4.4 \\
\nom & 51.4\,$\pm$\,6.3 & 68.9\,$\pm$\,8.6 & 52.8\,$\pm$\,11.4 & 75.9\,$\pm$\,10.8 \\
\CPO & 22.3\,$\pm$\,2.3 & 46.9\,$\pm$\,5.2 & 37.6\,$\pm$\,6.6 & 65.3\,$\pm$\,5.7 \\
\Lag & 63.0\,$\pm$\,12.9 & 67.6\,$\pm$\,18.2 & 37.2\,$\pm$\,6.3 & 69.1\,$\pm$\,5.4 \\[1pt]
\ETsf & 46.0\,$\pm$\,3.4 & 69.1\,$\pm$\,7.6 & 51.3\,$\pm$\,5.5 & 87.6\,$\pm$\,3.6 \\
\nomsf & 41.6\,$\pm$\,2.4 & 76.6\,$\pm$\,3.0 & 53.5\,$\pm$\,7.8 & 89.2\,$\pm$\,5.2 \\
\CPOsf & 30.6\,$\pm$\,2.3 & 71.4\,$\pm$\,4.4 & 40.5\,$\pm$\,8.6 & 84.3\,$\pm$\,7.6 \\
\Lagsf & 49.5\,$\pm$\,5.0 & 75.7\,$\pm$\,5.4 & 44.6\,$\pm$\,7.9 & \textbf{91.5}\,$\pm$\,2.2 \\
\bottomrule
\end{tabular}}}%
\usebox0
\par\vspace{3pt}%

\parbox{\columnwidth}{%
\scriptsize\raggedright
Each row corresponds to 31200 games.
WR: win rate.
SR: safe rate.}
\end{table}

\subsection{Hardware Stress Tests Against Human Opponents}
\label{sec:hardware}
For the real-world evaluation, we consider an asymmetric attack--defense game where the attacker needs to reach a target line within a time limit to win. This asymmetry induces more intense interactions, such as frequent close contact between the two players.
The asymmetric game is trained and evaluated in simulation under the same
protocol as the symmetric one, where \StoC{}  attains both the highest win rate and
the highest safe rate (Table~\ref{tab:main}). We further conduct this game in
hardware experiments and evaluate the policies against a human opponent.

\para{Hardware System Setup}
We use an HTC Vive Lighthouse system to localize both the Unitree Go2 quadrupedal robot and the human.
This system consists of two base stations placed at the diagonal corners of the field and two SteamVR trackers mounted on the robot and its human opponent. 
Policy inference and the safety filter run on the Jetson Nano at $50$\,Hz. All five methods are trained under the same domain-randomization strategy. We evaluate \StoC{} and all four baselines for ten rounds in each attacker/defender role.

\para{Results}
\begin{table}[t]
\centering
\footnotesize
\setlength{\tabcolsep}{3pt}
\renewcommand{\arraystretch}{1.05}
\caption{Hardware trials against a human opponent.}
\label{tab:hardware}

\begin{tabular*}{\columnwidth}{
@{\extracolsep{\fill}}lcccc@{}
}
\toprule
\multicolumn{5}{c}{\textbf{Robot attacking}}\\
\midrule
Method & TDR ($\%$) $\uparrow$ & SR ($\%$) $\uparrow$ & HPS (m$/$s) & HPA (m$/$s$^2$)\\
\midrule
\StoC
& \textbf{70.0} & \textbf{80.0} & $2.14 \pm 0.36$ & $5.83 \pm 1.30$\\
\ET
& 50.0 & 60.0 & $1.86 \pm 0.26$ & $3.17 \pm 1.03$\\
\nom
& 30.0 & 30.0 & $1.46 \pm 0.19$ & $2.35 \pm 1.23$\\
\CPO
& 0.0 & 0.0 & $1.30 \pm 0.38$ & $2.37 \pm 0.73$\\
\Lag
& 20.0 & 20.0 & $0.74 \pm 0.43$ & $1.04 \pm 0.58$\\

\midrule
\multicolumn{5}{c}{\textbf{Robot defending}}\\
\midrule
Method & TOR ($\%$) $\uparrow$ & SR ($\%$) $\uparrow$ & HPS (m$/$s) & HPA (m$/$s$^2$)\\
\midrule
\StoC
& \textbf{70.0} & 90.0 & $2.40 \pm 0.38$ & $4.31 \pm 1.14$\\
\ET
& 60.0 & 80.0 & $2.07 \pm 0.16$ & $4.34 \pm 1.21$\\
\nom
& 40.0 & 50.0 & $1.22 \pm 0.15$ & $2.46 \pm 0.79$\\
\CPO
& 0.0 & \textbf{100.0} & $1.30 \pm 0.22$ & $2.15 \pm 0.70$\\
\Lag
& 10.0 & 60.0 & $1.10 \pm 0.21$ & $1.85 \pm 0.66$\\
\bottomrule
\end{tabular*}

\vspace{3pt}
\parbox{\columnwidth}{%
\scriptsize\raggedright
Each row corresponds to ten trials.
TDR/TOR: touchdown/timeout rate.
SR: safe rate.
HPS: human's peak speed.
HPA: human's peak acceleration.
}
\end{table}
Table~\ref{tab:hardware} summarizes task completion, safe rates, and opponent kinematics across the hardware trials.
We record the following numbers as measures of human competence:
\begin{itemize}
    \item \textbf{Human peak speed (HPS).} The human opponent's peak speed in m$/$s.
    \item \textbf{Human peak acceleration (HPA).} The human opponent's peak acceleration in m$/$s$^2$.
\end{itemize}
Across both roles, \StoC{} achieves the highest task completion ($70\%$) while maintaining high safe rates ($80\%$ as attacker and $90\%$ as defender). 
In comparison, \ET{} and \nom{} compromise safety under competitive pressure, with safe rates dropping to $30\%$--$60\%$ when attacking. 
Meanwhile, \CPO{} fails to reconcile the two objectives: it causes a safety violation on every attacking trial (zero touchdowns), and achieves zero defensive violations only by passively traversing the field, unable to contest the human attacker in close proximity. 
From the HPS and HPA values, \StoC{} has to respond to the most demanding human movements with the highest mean peak speed in both roles, the highest attacking acceleration ($5.83\,\mathrm{m/s^2}$), and defending acceleration comparable to that of \ET{}. Those human strategies were not encountered during training. These trials provide evidence that \StoC{} is more robust against out-of-distribution opponents and support \textbf{H2}.

\section{Conclusion}
\label{sec:conclusion}

\para{Limitations}
While our analysis assumes perfect safety filters, deriving formal guarantees for $\nashTol$-Nash equilibria under approximation errors stemming from imperfect filters and task policies remains an open research question.
We also see an exciting opportunity to integrate S2C with latent-space safety filters with visual feedback~\cite{nakamura2025generalizing} for competitive robot interactions in unstructured, open-world environments.

\para{Summary}
In this paper, we introduced S2C, a two-stage framework that decouples safety from competitive task-policy learning in multi-robot interactions. We proved that under perfect safety filtering, equilibrium policies of the filtered game remain non-exploitable within the safe policy class of the original safety-critical Markov game. We validated S2C on a quadruped touchdown game in both simulation and hardware tests. The results demonstrate that S2C achieves the highest task performance and the lowest exploitability across the evaluation pool while maintaining high safe rates.
Our work establishes a principled framework for learning safe and minimally exploitable robot policies for competitive interactions.

\balance
\printbibliography 

\end{document}